\documentclass[11pt]{article} 

\usepackage{fullpage}
\usepackage[T1]{fontenc}
\usepackage{graphicx}

\usepackage{epstopdf}
\usepackage{amsmath}
\usepackage{amsthm}
\usepackage{algcompatible}
\usepackage{amssymb}
\usepackage{enumerate}

\usepackage{algorithm}
\usepackage{algpseudocode}
\usepackage{mathtools, cuted}
\usepackage{breqn}
\usepackage{bbm}
\usepackage{cite}

\graphicspath{{./figures/}}
\usepackage{pdfpages}
\usepackage{chngcntr}

\newtheorem{assumption}{\bf Assumption}
\newtheorem{remark}{\bf Remark}

\newtheorem{theorem}{\bf Theorem}
\newtheorem{lemma}{\bf Lemma}
\newtheorem{proposition}{\bf Proposition}
\newtheorem{problem}{\bf Problem}
\newtheorem{corollary}{\bf Corollary}

\newcommand{\abs}[1]{\ensuremath{\left\vert #1\right\vert}}
\newcommand{\norm}[1]{\ensuremath{\left\| #1\right\|}}
\newcommand{\snorm}[1]{\ensuremath{\| #1\|}}
\newcommand{\paren}[1]{\ensuremath{\left( #1\right)}}
\newcommand{\clint}[1]{\ensuremath{\left[ #1\right]}}

\newcommand{\set}[1]{\ensuremath{\left\{ #1\right\}}}

\newcommand{\matr}[1]{\ensuremath{\clint{\begin{array} #1 \end{array}}}}
\newcommand{\expe}[1]{\ensuremath{\mathbb{E}\clint{#1}}}

\newcommand{\REAL}{\ensuremath{\mathbb{R}}}

\renewcommand{\P}{\ensuremath{\mathbb{P}}}

\renewcommand{\O}{\ensuremath{\mathcal{O}}}
\newcommand{\K}{\ensuremath{\mathcal{K}}}
\newcommand{\F}{\ensuremath{\mathcal{F}}}
\newcommand{\C}{\ensuremath{\mathcal{C}}}
\newcommand{\T}{\ensuremath{\mathcal{T}}}
\newcommand{\E}{\ensuremath{\mathcal{E}}}
\newcommand{\Tm}{\ensuremath{\mathcal{T}}}

\DeclareMathOperator{\Tr}{\mathrm{tr}}

\DeclarePairedDelimiter{\diagfences}{(}{)}
\newcommand{\diag}{\operatorname{diag}\diagfences}

\usepackage{array}
\newcolumntype{"}{@{\hskip\tabcolsep\vrule width 2pt\hskip\tabcolsep}}

\begin{document}

\title{
	Finite Sample Analysis of Stochastic System Identification
}

\author{Anastasios~Tsiamis $^\star$ and George~J.~Pappas 
\thanks{The authors  are   with   the   Department   of   Electrical   and   Systems  Engineering,  University  of  Pennsylvania,  Philadelphia,  PA  19104.
		 Emails: \{atsiamis,pappasg\}@seas.upenn.edu}
}
\maketitle
\begin{abstract}
In this paper, we analyze the finite sample complexity of stochastic system identification using modern tools from machine learning and statistics. An unknown discrete-time linear system evolves over time under Gaussian noise without external inputs. The objective is to recover the system parameters as well as the Kalman filter gain, given a single trajectory of output measurements over a finite horizon of length $N$.
Based on a subspace identification algorithm and a finite  number of $N$ output samples, we provide non-asymptotic high-probability upper bounds
for the system parameter estimation errors.
Our analysis uses recent results from random matrix theory, self-normalized martingales and SVD robustness, in order to show that with high probability the estimation errors decrease with a rate of $1/\sqrt{N}$.  Our non-asymptotic bounds not only agree with classical asymptotic results, but are also valid even when the system is marginally stable.
\end{abstract}


\section{Introduction}\label{Section_Introduction}
Identifying predictive models from data has been a fundamental problem across several fields, from classical control theory to economics and modern machine learning.
System identification, in particular,
has a long history of studying this problem from a control theoretic perspective~\cite{ljung2010perspectives}. Identifying linear state-space models:
\begin{equation}\label{EQN_System_General}
\begin{aligned}
x_{k+1}&=Ax_{k}+Bu_k+w_{k}\\
y_{k}&=Cx_k+Du_k+v_k,
\end{aligned}
\end{equation}
from input-output data has been the focus of time-domain identification. 
In fact, some identification algorithms can not only learn the system matrices in~\eqref{EQN_System_General} but also the Kalman filter required for state estimation~\cite{van2012subspace}.

Most identification methods for linear systems either follow the  prediction error approach~\cite{Ljung1999system} or the subspace method~\cite{van2012subspace,verhaegen2007filtering}. The prediction error approach is usually non-convex and directly searches over the system parameters $A,B,C,D$ by minimizing a prediction error cost. The subspace approach is a convex one; first, Hankel matrices of the system are estimated, then, the parameters are realized via steps involving singular value decomposition (SVD). Methods inspired by machine learning have also also been employed~\cite{chiuso2019perspective}. 
In this paper, we focus on the subspace identification approach--see~\cite{qin2006overview} for an overview.
The convergence properties of subspace algorithms have been studied before in~\cite{deistler1995consistency,peternell1996statistical,viberg1997analysis,jansson1998consistency,bauer1999consistency,knudsen2001consistency}; the analysis relies on the assumption of asymptotic stability (spectral radius $\rho(A)<1$) and is limited to asymptotic results. In~\cite{deistler1995consistency,peternell1996statistical} it is shown that the identification error can decrease as fast as $\O\paren{1/\sqrt{N}}$ up to logarithmic factors, as the number of output data $N$ grows to infinity. While asymptotic results have been established, a finite sample analysis of subspace algorithms remains an open problem~\cite{van2012subspace}. Another open question is whether the asymptotic stability condition $\rho(A)<1$ can be relaxed to marginal stability $\rho(A)\le 1$.

From a machine learning perspective, finite sample analysis has been a standard tool for comparing algorithms in the non-asymptotic regime. 
Early work in finite sample analysis of system identification can be found in~\cite{weyer1999finite,campi2002finite,vidyasagar2008learning}.
 A series of recent papers~\cite{faradonbeh2018finite,simchowitz2018learning,sarkar2018fast} studied the finite sample properties of system identification from a single trajectory, when the system state is fully observed ($C=I$).
Finite sample results for partially observed systems ($C\neq I$), which is a more challenging problem, appeared recently in~\cite{oymak2018non,simchowitz2019semi,sarkar2019finite}. These papers provide a non-asymptotic convergence rate of $1/\sqrt{N}$ for the recovery of matrices $A,B,C,D$ up to a similarity transformation. The results rely on the assumption that that the system can be driven by external inputs, i.e. $B,D\neq 0$.  
In~\cite{oymak2018non}, the analysis of the classical Ho-Kalman realization algorithm was explored. In~\cite{simchowitz2019semi}, it was shown that with a prefiltering step, consistency can be achieved even for marginally stable systems where $\rho\paren{A}\le 1$.
In~\cite{sarkar2019finite}, identification of a system of unknown order is considered. 
Finite sample properties of system identification algorithms have also been used to design controllers~\cite{dean2017sample}. The dual problem of Kalman filtering has not been studied yet in this context.

In this paper, we perform the first finite sample analysis of system~\eqref{EQN_System_General} in the case $B,D=0$, when we have no inputs, also known as {\em stochastic system identification} (SSI)~\cite{van2012subspace}.  This problem is more challenging than the case $B,D\neq 0$, since the system can only be driven through noise and establishing persistence of excitation is harder.
We provide the first non-asymptotic guarantees for the estimation of matrices $A,C$  as well as the Kalman filter gain of~\eqref{EQN_System_General}.
Similar to~\cite{sarkar2018fast,oymak2018non}, the analysis is based on new tools from machine learning and statistics~\cite{vershynin2018high,abbasi2011improved,tu2016low}. 
In summary the main contributions of this paper are: 
\begin{itemize}
	\item To the best of our knowledge, our paper provides the first finite sample upper bounds in the case of stochastic system identification, where we have no inputs and the system is only driven by noise.
	We also provide the first finite sample guarantees for the estimation error of the Kalman filter gain. 
	
	\item  We prove that the outputs of the system satisfy persistence of excitation in finite time with high probability. This result is fundamental for the analysis of most subspace identification algorithms, which use outputs as regressors.

	\item We show that we can achieve a learning rate of $\O(\sqrt{1/N})$  up to logarithmic factors for marginally stable systems. To the best of our knowledge, the classical subspace identification results do not offer guarantees in the case of marginal stability where $\rho(A)\leq 1$.
	For stable systems ($\rho(A)<1$), this rate is consistent with classical asymptotic results~\cite{deistler1995consistency}. 
	
\end{itemize}
All proofs are included in the Appendix.

\section{Problem formulation}\label{Section_Formulation}
Consider the standard state space representation~\eqref{EQN_System_General} with $B,D=0$,
where $x_k\in\REAL^n$ is the system state, $y_k\in\REAL^{m}$ is the output, $A\in\REAL^{n\times n}$ is the system matrix, $C\in\REAL^{m\times n}$ is the output matrix, $w_k\in\REAL^{n}$ is the process noise and $v_k\in \REAL^{m}$ is the measurement noise. The noises $w_k,\,v_k$ are assumed to be i.i.d. zero mean Gaussian, with covariance matrices $Q$ and $R$ respectively, and  independent of each other. The initial state $x_0$ is also assumed to be zero mean  Gaussian, independent of the noises, with covariance $\Sigma_0$. Matrices $A$, $C$, $Q$, $R$, $\Sigma_0$ are initially unknown. However, the following assumption holds throughout the paper.

\begin{assumption}\label{ASS_Kalman}
	The order of the system $n$ is known. System $A$ is marginally stable: $\rho\paren{A}\le 1$, where $\rho$ denotes the spectral radius. The pair $(A,C)$ is observable, $(A,Q^{1/2})$ is controllable and $R$ is strictly positive definite. \hfill $\diamond$
\end{assumption}

We leave the case of unknown $n$ for future work \footnote{The results of Section~\ref{Section_Analysis_Hankel} do not depend on the order $n$ being known.}. The assumption $\rho(A)\le 1$ is more general than the stricter condition $\rho(A)<1$ found in previous works,  see~\cite{deistler1995consistency,peternell1996statistical,viberg1997analysis,jansson1998consistency,bauer1999consistency,knudsen2001consistency}. 
The remaining conditions in Assumption~\ref{ASS_Kalman}  are standard for the stochastic system  identification problem to be well posed and the Kalman filter to converge. 

The steady-state Kalman filter of system~\eqref{EQN_System_General} is :
\begin{equation}\label{EQN_System_Innovation}
\begin{aligned}
\hat{x}_{k+1}&=A\hat{x}_{k}+Ke_{k}\\
y_{k}&=C\hat{x}_k+e_k,
\end{aligned}
\end{equation}
where $K\in \REAL^{n\times m}$ is the steady-state Kalman gain
\begin{equation}\label{EQN_Kalman_Gain}
K=APC^*(CPC^*+R)^{-1}
\end{equation}
and $P$ is the positive definite solution of the  Riccati equation:
\begin{equation*}
P=APA^*+Q-APC^*(CPC^*+R)^{-1}CPA^*
\end{equation*}
A byproduct of Assumption~\ref{ASS_Kalman} is that the closed-loop matrix $A-KC$ has all the eigenvalues strictly inside the unit circle~\cite{anderson2005optimal}.
The state of the Kalman filter is the minimum mean square error prediction:
\begin{equation}\label{EQN_State_Estimation}
\hat{x}_{k+1}=\expe{x_{k+1}|y_0,\dots,y_k},\,\hat{x}_{0}=0.
\end{equation}
We denote its covariance matrix by:
\begin{equation}\label{EQN_State_Prediction_Covariance}
\Gamma_{k}=\expe{\hat{x}_{k}\hat{x}^*_{k}}
\end{equation}
The innovation error sequence $e_k$ has covariance
\begin{equation}\label{EQN_Innovations}
\bar{R}\triangleq\expe{e_ke_k^*}= CPC^*+R
\end{equation}
Since the original errors are Gaussian i.i.d., by the orthogonality principle the innovation error sequence $e_k$ is also Gaussian and i.i.d.
The later property is true since we also assumed that the Kalman filter is in steady-state.
\begin{assumption}\label{ASS_prediction_error_stationary}
	We assume that $\Sigma_0=P$, so that the Kalman filter~\eqref{EQN_System_Innovation} has converged to its steady-state. \hfill $\diamond$
\end{assumption}

Since the Kalman filter converges exponentially fast to the steady-state gain, this assumption is reasonable in many situations; it is also standard~\cite{deistler1995consistency,knudsen2001consistency}. Nonetheless, we leave the case of general $\Sigma_0$ for future work.

In the classical stochastic subspace identification problem, the main goal is to identify the Kalman filter parameters $A,C,K$ from output samples $y_0\dots,y_{N}$--see for example Chapter~3 of~\cite{van2012subspace}. 
The problem is ill-posed in general since the outputs are invariant under any similarity transformation $\bar{A}=S^{-1}AS$, $\bar{C}=CS$, $\bar{K}=S^{-1}K$. Thus, we can only estimate $A, C,K$ up to a similarity transformation.

In this paper, we will analyze the finite sample properties of a subspace identification algorithm, which is based on least squares. 
\begin{problem}[Finite Sample Analysis of SSI]\label{Problem}
	Consider a finite number $N$ of output samples $y_{0},\dots,y_{N-1}$, which follow model~\eqref{EQN_System_General} with $B,D=0$, and an algorithm $\mathcal{A}$, which returns estimates $\hat{A},\hat{C},\hat{K}$ of the true parameters. 
	Given a confidence level $\delta$ provide upper bounds $\epsilon_A\paren{\delta,N}$, $\epsilon_C\paren{\delta,N}$, $\epsilon_K\paren{\delta,N}$ and an invertible matrix $S$ such that with probability at least $1-\delta$:
	\begin{equation}\label{EQN_Problem_Bound_General}
	\begin{aligned}
	&\norm{\hat{A}-S^{-1}AS}_2\le \epsilon_A\paren{\delta,N}\\
	&\norm{\hat{C}-CS}_2\le \epsilon_C\paren{\delta,N}\\
	&\norm{\hat{K}-S^{-1}K}_2\le \epsilon_K\paren{\delta,N},
	\end{aligned}
	\end{equation}
	where $\norm{\cdot}_2$ denotes the spectral norm.
	The bounds $\epsilon$ can also depend on the model parameters $n,A,C,R,Q$ as well as the identification algorithm used. 
	\hfill $\diamond$
\end{problem}

 \section{Subspace Identification Algorithm}\label{Section_Identification_Algorithm}
The procedure of estimating the parameters $A,C,K$ is based on a least square approach, see for example~\cite{deistler1995consistency,knudsen2001consistency}. It involves two stages. First, we regress future outputs to past outputs to obtain a Hankel-like matrix, which is a product of an observability and a controllability matrix. Second, we perform a balanced realization step, similar to the Ho-Kalman algorithm, to obtain estimates for $A,C,K$.

Before describing the algorithm, we need some definitions.
Let $p,f$, with $p,f\ge n$ be two design parameters that define the horizons of the past and the future respectively. Assume that the total number of output samples is $\bar{N}=N+p+f-1$. Then, the future outputs $Y^{+}_{k}\in \REAL^{mf}$ and past outputs $Y^{-}_k\in\REAL^{mp}$ at time $k\ge p$ are defined as follows:
\begin{align}\label{EQN_future_past_output}
Y^{+}_{k}&\triangleq \matr{{c}y_{k}\\\vdots\\y_{k+f-1}},\quad
Y^{-}_{k}\triangleq\matr{{c}y_{k-p}\\\vdots\\y_{k-1}},\,k\ge p
\end{align}
By stacking the outputs for all sample sequences, over all times $p\le k\le N+p-1$, we form the batch outputs:
\begin{align*}
Y_{+}&\triangleq\matr{{ccc}Y^{+}_{p}&\dots&Y^{+}_{N+p-1}},\\
Y_{-}&\triangleq\matr{{ccc}Y^{-}_{p}&\dots&Y^{-}_{N+p-1}},
\end{align*}
The past and future noises $E^{+}_k,E^{-}_k,E_{+},E_{-}$ are defined similarly.  Finally, define the batch states:
\begin{align*}
\hat{X}&\triangleq\matr{{ccc}\hat{x}_{0}&\dots&\hat{x}_{N-1}}
\end{align*}
	The (extended) observability matrix $\O_k\in \REAL^{mk\times n}$ and the reversed (extended) controllability matrix $\K_k\in \REAL^{n\times mk}$ associated to system~\eqref{EQN_System_Innovation} are defined as:
\begin{equation}\label{EQN_Observability_Matrix}
\O_k\triangleq\matr{{cccc}C^*&A^*C^*&\cdots&(A^*)^{k-1}C^*}^*,
\end{equation}
\begin{equation}
\label{EQN_Kalman_Controllability}
\K_k\triangleq\matr{{cccc}(A-KC)^{k-1}K&\dots&(A-KC)K&K}
\end{equation}
respectively.
We denote the Hankel(-like) matrix $\O_f\K_p$ by:
\begin{equation}\label{EQN_Hankel}
G\triangleq\O_f \K_p.
\end{equation}
Finally, for any $s\ge 2$, define block-Toeplitz matrix:
\begin{equation}
\label{POE_EQN_Innovation_Toeplitz}
\Tm_s\triangleq\matr{{cccc}I_m&0& &0\\CK&I_m&\cdots&0\\ \vdots&\vdots& &\vdots \\CA^{s-2}K&CA^{s-3}K&\cdots&I_m}. 
\end{equation}

  \subsection{Regression for Hankel Matrix Estimation}
First, we establish a linear regression between the future and past outputs. This step is common in most (stochastic) subspace identification algorithms. From~\eqref{EQN_System_Innovation}, for every $k$:
\[
Y^{+}_k=\O_f \hat{x}_{k}+\Tm_f E^{+}_k.
\]
Meanwhile, from~\eqref{EQN_System_Innovation}, the state prediction $\hat{x}_k$ can be expressed in terms of the past outputs:
\[
\hat{x}_k=Ky_{k-1}+\dots+(A-KC)^{p-1}Ky_{k-p}+(A-KC)^{p}\hat{x}_{k-p}.
\]
After some algebra, we derive the linear regression:
\begin{equation}\label{EQN_basic_regression}
Y_{+}=GY_{-} + \O_f(A-KC)^p\hat{X}+ \Tm_f E_{+},
\end{equation}
where the regressors $Y_-$ and the residuals $E_{+}$ are independent column-wise. The term $\O_k(A-KC)^p\hat{X}$ introduces a bias due to the Kalman filter truncation, where we use only $p$ past outputs instead of all of them. 
Based on~\eqref{EQN_basic_regression}, we compute the least squares estimate \begin{equation}\label{EQN_G}
\hat{G}=Y_{+}Y_{-}^*(Y_-Y_-^*)^{-1}.
\end{equation} 
The Hankel matrix $G$ can be interpreted as a (truncated) Kalman filter which predicts future outputs directly from past outputs, independently of the internal state-space representation~\cite{van2012subspace}. In this sense, the estimate $\hat{G}$ is a ``data-driven" Kalman filter. Notice that persistence of excitation of the outputs (invertibility of $Y_-Y^*_-$) is required in order to compute the least squares estimate $\hat{G}$.

\subsection{Balanced Realization} \label{Subsection_Matrix_Estimation}
This step determines a balanced realization of the state-space, which is only one of the possibly infinite state-space representations--see Section~\ref{Section_Conclusion} for comparison with other subspace methods.
First, we compute a rank-$n$ factorization of the full rank matrix $\hat{G}$. 
Let the SVD of $\hat{G}$ be:
\begin{equation}\label{EQN_SVD}
\hat{G}=\matr{{cc}\hat{U}_1&\hat{U}_2}\matr{{cc}\hat{\Sigma}_1&0\\0&\hat{\Sigma}_2}\matr{{c}\hat{V}_1^*\\\hat{V}_2^*},
\end{equation}
where $\hat{\Sigma}_1\in\REAL^{n\times n}$ contains the $n-$largest singular values.  
 Then, a standard realization of $\O_f$, $\K_p$ is:
\begin{equation}\label{EQN_realization_observability_controllability}
\hat{\O}_{f}=\hat{U}_1\hat{\Sigma}^{1/2}_1,\: \hat{\K}_{p}=\hat{\Sigma}^{1/2}_1\hat{V}^*_1.
\end{equation}

This step assumed knowing the order $n$ of the system, see Assumption~\ref{ASS_Kalman}. In addition, matrix $\K_p$ should have full rank $n$. This is equivalent to the pair $(A,K)$ being controllable. Otherwise, $\O_f\K_p$ will have rank less than $n$ making it impossible to accurately estimate $\O_f$.
\begin{assumption}\label{ASS_Subspace}
The pair $(A,K)$ is controllable.  \hfill $\diamond$
\end{assumption}
The above assumption is standard--see for example~\cite{knudsen2001consistency}.

Based on the estimated observability/controllability matrices, we can approximate the system parameters as follows:
\begin{equation*}\label{EQN_C_K_estimate}
\hat{C}=\hat{\O}_f\paren{1:m,:},\quad\hat{K}=\hat{\K}_p\paren{:,(p-1)m+1:pm},
\end{equation*} 
where the notation $\hat{\O}_f\paren{1:m,:}$ means we pick the first $m$ rows and all columns. The notation for $\hat{\K}_p$ has similar interpretation.
For simplicity, define 
\[\hat{\O}^{u}_{f}\triangleq \hat{\O}_{f}\paren{1:m(f-1),:},
\]
 which includes the $m(f-1)$ ``upper" rows of matrix $\hat{\O}_{f}$. Similarly, we define the lower part $\hat{\O}^{l}_{f}$.
 For matrix $A$ we exploit the structure of the extended observability matrix and solve
$
\hat{\O}^{u}_{f}\hat{A}=\hat{\O}^{l}_p
$
in the least squares sense by computing
\begin{equation*}\label{EQN_A_estimate}
\hat{A}=\paren{\hat{\O}^{u}_{f}}^{\dagger}\hat{\O}^{l}_p,
\end{equation*} 
where $\dagger$ denotes the pseudoinverse.

Now that we have a stochastic system identification algorithm, our goal is to perform finite sample analysis, 
which is divided in two parts. 
First, in Section~\ref{Section_Analysis_Hankel}, we provide high probability upper bounds for the error $\snorm{G-\hat{G}}_2$ in the regression step. Then, in Section~\ref{Section_Analysis_Matrices}, we analyze the robustness of the balanced realization step.

\section{Finite Sample Analysis of Regression}\label{Section_Analysis_Hankel}
In this section, we provide the finite sample analysis of the linear regression step of the identification algorithm. We provide high-probability upper bounds for estimation error $\snorm{G-\hat{G}}_2$ of the Hankel-like matrix $G$. 
Before we state the main result, let us introduce some notation.
Recall the definition of the innovation covariance matrix $\bar{R}$ in~\eqref{EQN_Innovations}.
We denote the past noises' weighted covariance by:
\begin{equation}\label{MAIN_EQN_Past_Noise_Covariance}
\Sigma_{E}=\expe{\Tm_pE^{-}_k(E^{-}_k)^*\Tm_p^*}=\Tm_p\diag{\bar{R},\dots,\bar{R}}\Tm_p^*.
\end{equation}
The least singular value of the above matrix is denoted by:
\begin{equation}\label{MAIN_EQN_Past_Noise_Sigma_Min}
\sigma_{E}\triangleq \sigma_{\min}\paren{\Sigma_E}.
\end{equation}
In~Lemma~\ref{POE_LEM_Sigma_E} in the Appendix, we show that $\sigma_E\ge\sigma_{\min}\paren{R}>0$.

\begin{theorem}[Regression Step Analysis]\label{MAIN_THM_Hankel_Matrix}
	Consider system~\eqref{EQN_System_Innovation} under the Assumptions~\ref{ASS_Kalman},~\ref{ASS_prediction_error_stationary},~\ref{ASS_Subspace}. Let $\hat{G}$ be the estimate~\eqref{EQN_G} of the subspace identification algorithm given an output trajectory $y_0,\dots,y_{N+p+f-1}$ and let $G$ be as in~\eqref{EQN_Hankel}. Fix a confidence $\delta>0$ and define:
		\begin{equation}\label{MAIN_EQN_DELTA_N}
\delta_N\triangleq\paren{2(N+p-1)m}^{-\log^2\paren{2pm}\log\paren{2(N+p-1)m}}.
	\end{equation}
	 There exist $N_0,N_1, N_2$ such that if $N\ge N_0, N_1, N_2$, (see definitions~\eqref{POE_EQN_N0},~\eqref{POE_EQN_N1},~\eqref{TRUNC_EQN_N2} in the Appendix), then with probability at least $1-\delta_N-6\delta$:
	\begin{align} \label{MAIN_EQN_Basic_Theorem_Statement}
	\norm{G-\hat{G}}_2\le & \underbrace{\mathcal{C}_1\sqrt{\frac{fmp}{N}\log\frac{5f\kappa_{N}}{\delta}}}_{\O\paren{\sqrt{p\log N/N}}} +\underbrace{\mathcal{C}_2\norm{(A-KC)^p}_2}_{\O\paren{\rho\paren{A-KC}^p}},
	\end{align}
	where 
	\begin{equation}\label{MAIN_EQN_Condition_Number_Approx}
	\kappa_{N}=\frac{4}{\sigma_{E}}\paren{\norm{\O_p}^2_2\Tr\Gamma_{N-1}+\Tr\Sigma_{E}}+\delta
\end{equation}	
over-approximates the condition number of $\expe{Y_-Y_-^*}$ and
		\begin{equation}\label{MAIN_EQN_System_Specific_Constant}
	\mathcal{C}_1= 8\sqrt{\frac{\norm{\bar{R}}_2}{\sigma_E}}\norm{\Tm_f}_2,\quad \mathcal{C}_2=4\norm{\O_f}_2\norm{\O^{\dagger}_p}_2,
	\end{equation}
	are system-dependent constants.
	\hfill $\diamond$
\end{theorem}
The final result in~\eqref{MAIN_EQN_Basic_Theorem_Statement} has been simplified for compactness--see~\eqref{Proof_EQN_Basic_Result_Equivalent} for the full expression in the Appendix.

\begin{remark}[Result interpretation]
From~\eqref{EQN_basic_regression},~\eqref{EQN_G} the estimation error consists of two terms:
\begin{align}
\hat{G}-G=\underbrace{\Tm_f E_{+}Y^*_{-}\paren{Y_{-}Y_{-}^*}^{-1}}_{\text{Cross term}}+\underbrace{\O_f\paren{A-KC}^p\hat{X}Y^*_{-}\paren{Y_{-}Y^*_{-}}^{-1}}_{\text{Kalman filter truncation bias term}}\label{MAIN_EQN_Interpretation}.
\end{align} 
The first term in~\eqref{MAIN_EQN_Basic_Theorem_Statement} corresponds to the cross-term error, while the second term corresponds to the Kalman filter truncation bias term. 
To obtain consistency for $\hat{G}$, we have to let the term $\norm{(A-KC)^p}_2$ go to zero with $N$. Recall that the matrix $A-KC$ has spectral radius less than one, thus, the second term decreases exponentially with $p$. By selecting $p=c \log N$, for some $c$, we can force the Kalman truncation error term to decrease at least as fast as the first one, see for example~\cite{deistler1995consistency}. In this sense, the dominant term is the first one, i.e. the cross-term. Notice, that $f$ can be kept bounded as long as it is larger than $n$. Notice that the norm $\norm{\O^{\dagger}_p}$ remains bounded by $\norm{\O^{\dagger}_n}$ as $p$ increases--see Lemma~\ref{ALG_LEM_Singular_Value_Lower_Bound} in Appendix. \hfill $\diamond$	
\end{remark}

\begin{remark}[Statistical rates]
For \textbf{marginally stable} systems ($\rho(A)\le 1$) and $p=c\log N$, we have $\log\kappa_{N}=\O\paren{\log N}$, since $\norm{\O_p}_2,\Tr \Gamma_{N}$ depend at most polynomially on $p,N$--see Corollary~\ref{ALG_COR_System_Upper_Bound} in the Appendix.
In this case, 
~\eqref{MAIN_EQN_Basic_Theorem_Statement} results in a rate of:
\[
\norm{G-\hat{G}}_2=\O\paren{\frac{\log N}{\sqrt{N}}}.
\]
To the best of our knowledge, there have not been any bounds for the performance of subspace algorithms in the case of marginally stable systems.

In the case of \textbf{asymptotically stable} systems ($\rho(A)< 1$), we have $\kappa_{N}=\O\paren{p}$, since $\norm{\O_p}_2,\Tr \Gamma_{N}, \norm{\Tm_p}_2$ are now $\O(1)$--see Corollary~\ref{ALG_COR_System_Upper_Bound} in the Appendix. Hence, if $p=c\log N$, we obtain a rate of:
\[
\norm{G-\hat{G}}_2=\O\paren{\sqrt{\frac{\log N\log\log N}{N}}}.
\]
As a result, our finite sample bound~\eqref{MAIN_EQN_Basic_Theorem_Statement} is consistent with the asymptotic bound in equation~(14) of~\cite{deistler1995consistency} as $N$ grows to infinity.\hfill $\diamond$
\end{remark}

In the absence of inputs ($B,D=0$), the noise both helps and obstructs identification. Larger noise leads to better excitation of the outputs, but also worsens the convergence of the least squares estimator. To see how our finite sample bounds capture that, observe that larger noise leads to bigger $\sigma_E$  but also bigger $\norm{\bar{R}}_2$. This trade-off is captured by~$\C_1$.

If $N$ is sufficiently large (condition $N\ge N_0,N_1$), the outputs are guaranteed to be persistently exciting in finite time; more details can be found in Section~\ref{Subsection_PE} and the Appendix. Meanwhile, condition $N\ge N_2$ is not necessary; it just leads to a simplified expression for the bound of the Kalman filter truncation error--see Section~\ref{Subsection_Kalman} and Appendix. The definitions of $N_0,N_1,N_2$ can be found in~\eqref{POE_EQN_N0},~\eqref{POE_EQN_N1},~\eqref{TRUNC_EQN_N2}. Their existence is guaranteed even if $p$ varies slowly with $N$, i.e. logarithmically.

Obtaining the bound on the error $\snorm{G-\hat{G}}_2$ in~\eqref{MAIN_EQN_Basic_Theorem_Statement} of Theorem~\ref{MAIN_THM_Hankel_Matrix} requires the following three steps:
\begin{enumerate}
	\item Proving persistence of excitation (PE) for the past outputs, i.e. invertibility of $Y_-Y_-^*$.
	\item Establishing bounds for the cross-term error in~\eqref{MAIN_EQN_Interpretation}.
	\item Establishing bounds for the  the Kalman truncation error in~\eqref{MAIN_EQN_Interpretation}.
\end{enumerate} 
In the following subsections, we sketch the proof steps. 

\subsection{Persistence of Excitation in Finite Time}\label{Subsection_PE}
The next theorem shows that with high probability the past outputs and noises are persistently exciting in finite time. The result is of independent interest and is fundamental since most subspace algorithms use past outputs as regressors.
\begin{theorem}[Persistence of Excitation]\label{MAIN_THM_PE}
Consider the conditions of Theorem~\ref{MAIN_THM_Hankel_Matrix} and $N_0$, $N_1$ as in~\eqref{POE_EQN_N0},~\eqref{POE_EQN_N1}. If $N\ge N_0,N_1$, then with probability at least $1-\delta_N-2\delta$ both of the following events occur:
\begin{align}\label{MAIN_EQN_Event_Output_PE}
\E_{Y}&=\set{Y_-Y_-^*\succeq \frac{1}{2}\O_p\hat{X}\hat{X}^*\O^*_p+\frac{1}{2}\Tm_p E_-E^*_-\Tm^*_p}\\ \label{MAIN_EQN_Event_Noise_PE}
\E_E&=\set{\Tm_pE_{-}E^*_{-}\Tm^*_p\succeq \frac{N}{2}\Sigma_{e}},
\end{align}
where $\succeq$ denotes comparison in the positive semidefinite cone.
Hence, with probability at least $1-\delta_N-2\delta$ the outputs satisfy the PE condition:
\begin{equation*}
Y_-Y_-^*\succeq \frac{N}{4}\sigma_{E}I_{mp},
\end{equation*}
where $\sigma_E>0$ is defined in~\eqref{MAIN_EQN_Past_Noise_Sigma_Min}.
\hfill $\diamond$
\end{theorem}
The above result implies that if the past noises satisfy a PE condition, then PE for the outputs is also guaranteed; the noises are the only way to excite the system in the absence of control inputs. 
The prove PE, from~\eqref{EQN_System_Innovation}, the past outputs satisfy:
\[
Y_{-}=\O_p \hat{X}+\Tm_pE_{-}
\]
Thus, their correlations are
\begin{align}
Y_{-} Y^*_{-}=&\O_p \hat{X}\hat{X}^*\O_p^*+\Tm_pE_{-}E^*_{-}\Tm^*_p+\nonumber\\&\O_p \hat{X}E^{*}_-\Tm^*_p+\Tm_pE_{-}\hat{X}^*\O_p^*\label{MAIN_EQN_Output_Correlations_Identity}.
\end{align}
We can first show PE for the noise correlations $\Tm_pE_{-}E^*_{-}\Tm^*_p$, i.e. show that the event $\E_E$ occurs with high probability when $N$ is sufficiently large (condition $N\ge N_0$). This behavior is due to the fact that $\expe{\Tm_pE_{-}E^*_{-}\Tm^*_p}=N\Sigma_E$ and the sequence $E^{-}_k$ is component-wise i.i.d. To prove this step, we use Lemma~C.2  from~\cite{oymak2018non}--see Lemma~\ref{POE_LEM_Past_Noises} in the Appendix.  

Meanwhile, the cross terms $\hat{X}E^{*}$  are much smaller and their norm increases with a rate of at most $\O(\sqrt{N})$ up to logarithmic terms. This is since $\expe{\hat{X}E^{*}}=0$ and the product $\hat{X}E^{*}$ has martingale structure (see Appendix and Theorem~\ref{MART_THM_Vector} below).
Eventually, if the number of samples $N$ is large enough (condition $N\ge N_1$), the cross-terms will be dominated by the noise and state correlations with high probability, which establishes output PE.

\subsection{Cross-term error}\label{Subsection_Cross}
To bound the cross-term error, we express it as a product of $E_+Y^*_-\paren{Y_-Y_-^*}^{-1/2}$ and $(Y_-Y_-^*)^{-1/2}$, as in~\cite{sarkar2018fast}. The second term of the product can be bounded by applying Theorem~\ref{MAIN_THM_PE}. 
The first term is self-normalized and has martingale structure component-wise. 
In particular, the product $Y_-E^*_+$ is equal to:
\begin{equation*}
Y_-E^*_+=\matr{{ccc}\displaystyle\sum_{k=p}^{N+p-1}Y^-_{k}e^*_{k}&\dots& \displaystyle\sum_{k=p}^{N+p-1}Y^-_{k}e^*_{k+f-1}},
\end{equation*}
where every sum above is a martingale.
To bound it, we apply the next theorem, which generalizes Theorem~1 in~\cite{abbasi2011improved} and Proposition~8.2 in~\cite{sarkar2018fast}.
\begin{theorem}[Cross terms]\label{MART_THM_Vector}
	Let $\set{\F_t}_{t=0}^{\infty}$ be a filtration. Let $\eta_{t}\in\REAL^m$, $t\ge 0$ be $\F_t$-measurable, independent of $\F_{t-1}$. Suppose also that $\eta_{t}$ has independent components $\eta_{t,i}$ $i=1,\dots,m$, which are $1-$sub-Gaussian:
	\begin{equation*}
	\expe{e^{\lambda \eta_{t,i}}|\F_{t-1}}=\expe{e^{\lambda \eta_{t,i}}}\le e^{\lambda^2/2},\text{ for all }\lambda\in \REAL.
	\end{equation*}
	Let $X_{t}\in\REAL^{d}$, $t\ge 0$ be $\F_{t-1}-$measurable. Assume that $V$ is a $d\times d$ positive definite matrix. For any $t\ge 0$, define:
	\[
	\bar{V}_t=V+\sum_{s=1}^{t}X_sX_s^*,\qquad S_t=\sum_{s=1}^{t} X_sH^*_s,
	\]
	where
	\[
	H^{*}_s=\matr{{ccc}\eta^*_s&\dots&\eta^*_{s+r-1}}\in\REAL^{rm},
	\]
	for some integer $r$.
	Then, for any $\delta>0$, with probability at least $1-\delta$, for all $t\ge 0$
	\begin{equation*}
	\norm{\bar{V}_t^{-1/2} S_t }^2_2\le 8r\paren{\log\frac{r5^m}{\delta}+\frac{1}{2}\log\det\bar{V}_tV^{-1}}
	\end{equation*}
	\hfill $\diamond$
\end{theorem}
The above theorem along with a Markov upper bound on $Y_-Y_-^*$ (see Lemma~\ref{POE_LEM_upper_bounds} in the Appendix) are used to bound $E_+Y^*_-\paren{Y_-Y_-^*}^{-1/2}$. 
\subsection{Kalman truncation error}\label{Subsection_Kalman}
For the Kalman truncation error term, we need to bound the term $\hat{X}Y^*_-(Y_-Y^*_-)^{-1}$, which is $\O\paren{1}$.
Using the identities $\O_p^{\dagger}\O_p\hat{X}=\hat{X}$, and
$Y_{-}=\O_p\hat{X}+\Tm_pE_-$,
we derive the following equality:
\begin{align}
\hat{X}Y^*_-(Y_-Y^*_-)^{-1}=\O_p^{\dagger}&\paren{I_{mp}-\Tm_pE_{-}E_{-}^*\Tm^*_p(Y_-Y^*_-)^{-1}\right.\nonumber\\&\left.-\Tm_pE_{-}\hat{X}^*\O_p^*(Y_-Y^*_-)^{-1}}\label{MAIN_EQN_Kalman_Truncation}
\end{align}
From Theorem~\ref{MAIN_THM_PE}, we obtain $\norm{\Tm_pE_{-}E_{-}^*\Tm^*_p(Y_-Y^*_-)^{-1}}_2\le 2$.
The last term in~\eqref{MAIN_EQN_Kalman_Truncation}
can be treated like the cross-term in Section~\ref{Subsection_Cross}, by applying Theorems~\ref{MAIN_THM_PE},~\ref{MART_THM_Vector} and Lemma~\ref{POE_LEM_upper_bounds}. It decreases with a rate of $\O\paren{1/\sqrt{N}}$ up to logarithmic terms, so it is much smaller than the other terms in~\eqref{MAIN_EQN_Kalman_Truncation}. 
To keep the final bound simple, we select $N_2$ such that 
\begin{equation}\label{MAIN_EQN_Kalman_Cross}
\norm{\Tm_pE_{-}\hat{X}^*\O_p^*(Y_-Y^*_-)^{-1}}_2\le 1
\end{equation}
with high probability--see also~\eqref{TRUNC_EQN_N2} for the definition of $N_2$. 

\section{Robustness of Balanced Realization}\label{Section_Analysis_Matrices}
In this section, we analyze the robustness of the balanced realization. In particular, we upper bound the estimation errors of matrices $A,C,K$ in terms of the estimation error $\snorm{G-\hat{G}}_{2}$ obtained by Theorem~\ref{MAIN_THM_Hankel_Matrix}. 

Assume that we knew $G$ exactly. Then, the SVD of the true $G$, would be:
\begin{equation*}
G=\matr{{cc}U_1&U_2}\matr{{cc}\Sigma_1&0\\0&0}\matr{{c}V_1^*\\V_2^*}=U_1\Sigma_1V^*_1,
\end{equation*}
for some $\Sigma_1\in \REAL^{n\times n}$.
Hence, if we knew $G$ exactly, the output of the balanced realization would be:
\begin{equation}
\bar{\O}_f=U_1\Sigma^{1/2}_1,\,\bar{\K}_p=\Sigma^{1/2}_1V^*_1.
\end{equation}
The respective matrices $\bar{C},\bar{K},\bar{A}$ are defined similarly, based on $\bar{\O}_f,\,\bar{\K}_p$, as described in Section~\ref{Section_Identification_Algorithm}.
The system matrices $\bar{C},\bar{K},\bar{A}$ are equivalent to the original matrices $C,K,A$ up to a similarity transformation $\bar{C}=CS$, $\bar{K}=S^{-1}K$, $\bar{A}=S^{-1}AS$ for some invertible $S$.
For simplicity, we will quantify the estimation errors in terms of the similar $\bar{A},\bar{C},\bar{K}$ instead of the original $A,C,K$.

 The next result follows the steps of~\cite{oymak2018non} and relies on Lemma~5.14 of~\cite{tu2016low} and Theorem~4.1 of~\cite{wedin1973pseudo}. Let $\sigma_{n}\paren{\cdot}$ denote the $n-$th largest singular value. 
\begin{theorem}[Realization robustness]\label{SVD_THM_Main}
Consider the true Hankel-like matrix $G$ defined in~\eqref{EQN_Hankel} and the noisy estimate $\hat{G}$ defined in~\eqref{EQN_G}. Let $\hat{A},\hat{C},\hat{K},\hat{\O}_f$, $\hat{\K}_p$ be the output of the balanced realization algorithm based on $\hat{G}$. Let $\bar{A},\bar{C},\bar{K},\bar{\O}_f$, $\bar{\K}_p$ be the output of the balanced realization algorithm based on the true $G$. If $G$ has rank $n$ and the following robustness condition is satisfied:
	\begin{equation}\label{SVD_EQN_robustness_condition}
	\norm{\hat{G}-G}_2\le \frac{\sigma_{n}\paren{G}}{4},
	\end{equation}
	there exists an orthonormal matrix $T\in\REAL^{n\times n}$ such that:
	\begin{align}\label{SVD_EQN_Bound_OBS_CONT}
&	\norm{\hat{\O}_f-\bar{\O}_fT}_{2}\le 2\sqrt{\frac{10n}{\sigma_{n}\paren{G}}}\norm{G-\hat{G}}_2\nonumber\\
&	\norm{\hat{C}-\bar{C}T}_{2}\le \norm{\hat{\O}_f-\bar{\O}_fT}_{2}\nonumber\\
&	\norm{\hat{A}-T^*\bar{A}T}_{2}\le \underbrace{\frac{\sqrt{\norm{G}_2}+\sigma_o}{\sigma^2_o}}_{\O\paren{1}}\norm{\hat{\O}_f-\bar{\O}_fT}_{2}\nonumber\\
&	\norm{\hat{K}-T^*\bar{K}}_{2}\le 2\sqrt{\frac{10n}{\sigma_{n}\paren{G}}}\norm{G-\hat{G}}_2,\nonumber
	\end{align}
where
\[
\sigma_o=\min\paren{\sigma_{n}\paren{\hat{\O}^{u}_{f}},\sigma_{n}\paren{\bar{\O}^{u}_{f}}}.
\]
The notation $\hat{\O}^{u}_{f},\bar{\O}^{u}_{f},$ refers to the upper part of the respective matrix (first $(f-1)m$ rows)--see Section~\ref{Subsection_Matrix_Estimation}.
	\hfill $\diamond$
\end{theorem}
\begin{remark}
The result states that if the error of the regression step is small enough, then the realization is robust. 
The singular value $\sigma_n\paren{G}$ can be quite small. Hence, the robustness condition~\eqref{SVD_EQN_robustness_condition} can be quite restrictive in practice. 
However, such a condition is a fundamental limitation of the SVD procedure. The gap between the singular values of $G$ and the singular values coming from the noise $G-\hat{G}$ should be large enough; this guarantees that the singular vectors related to small singular values of $G$ are separated from the singular vectors coming from the noise $G-\hat{G}$, which can be arbitrary. See also 
 Wedin's theorem~\cite{wedin1972perturbation}. Such robustness conditions have also appeared in model reduction theory~\cite{pernebo1982model}.\hfill $\diamond$
\end{remark}

The term $\frac{\sqrt{\norm{G}_2}+\sigma_o}{\sigma^2_o}$ which appears in the bound of $A$ is $\O\paren{1}$. The value of $\sigma^{-1}_o$ is random since it depends on $\hat{\O}^h_f$. However, we could replace it by a deterministic bound. From:
\[
\sigma_{n}\paren{\hat{\O}^h_f}\ge \sigma_{n}\paren{\bar{\O}^h_f}-\norm{\hat{\O}^h_f-\bar{\O}^h_fT}_2\ge  \sigma_{n}\paren{\bar{\O}^h_f}-\norm{\hat{\O}_f-\bar{\O}_fT}_2,
\]
$\sigma_o$ will eventually be lower bounded by $\sigma_{n}\paren{\bar{\O}^h_f}/2$ if the error $\norm{\hat{\O}_f-\bar{\O}_fT}_2\le \sigma_{n}\paren{\bar{\O}^h_f}/2$ is small enough.  The norm $\snorm{G}_2\le \norm{\O_f}_2\norm{\K_p}_2$ is bounded as $p$ increases, since $A-KC$ is asymptotically stable and $f$ is finite.

\begin{remark}[Total bounds]
	The final upper bounds for the estimation of the system parameters $A,C,K$, as stated in Problem~\ref{Problem}, can be found by combining the finite sample guarantees of the regression step (Theorem~\ref{MAIN_THM_Hankel_Matrix}) with the robustness analysis of the realization step (Theorem~\ref{SVD_THM_Main}).
	All matrix estimation errors depend linearly on the Hankel matrix estimation error $\snorm{G-\hat{G}}_2$. As a result, all matrix errors have the same statistical rate as the error of $G$, i.e. their estimation error decreases at least as fast as $\O\paren{1/\sqrt{N}}$ up to logarithmic factors.\hfill $\diamond$
\end{remark}


\section{Discussion and Future Work}\label{Section_Conclusion}
One of the main differences between the subspace algorithm considered in this paper and other stochastic subspace identification algorithms is the SVD step. The other algorithms perform SVD on $W_1GW_2$ instead of $G$, where $W_1,W_2$ are full rank weighting matrices, possibly data dependent~\cite{van1995unifying,Ljung1999system,van2012subspace}. From this point of view, the results of Section~\ref{Section_Analysis_Hankel} (upper bound for $\snorm{G-\hat{G}}$ in~Theorem~\ref{MAIN_THM_Hankel_Matrix} and persistence of excitation in Theorem~\ref{MAIN_THM_PE}) are fundamental for understanding the finite sample properties of other subspace identification algorithms. Here, we studied the case $W_1=I,W_2=I$, which is not the standard choice~\cite{knudsen2001consistency}.
It is subject of future work to explore how the choice of $W_1,W_2$ affects the realization step, especially the robustness condition of the SVD step.

We could also provide finite sample bounds for the estimation of the closed-loop matrix $A_{c}\triangleq A-KC$. This can be done in two ways. One option is to form matrix $\hat{A}-\hat{K}\hat{C}$ from the estimates $\hat{A},\hat{C},\hat{K}$. Alternatively, we could estimate $\hat{A_{c}}$ directly from $\hat{\K}_p$ in the same way that we estimated $\hat{A}$. 
However, we do not have finite sample guarantees for the stability of the closed-loop estimates $\hat{A}-\hat{K}\hat{C}$, $\hat{A_c}$. It would be interesting to address this in future work.

Another direction for future work is repeating the analysis when the Kalman filter has not  reached steady-state, i.e. relax Assumption~\ref{ASS_prediction_error_stationary}. Finally, in this work, we only considered upper bounds. It would be interesting to study lower bounds as well to evaluate the tightness of our upper bounds. In any case, from lower bounds for fully observed systems~\cite{simchowitz2018learning}, the factor of $1/\sqrt{N}$ is tight.

\bibliographystyle{IEEEtran}
\bibliography{IEEEabrv,Literature_Identification}
\appendix
\counterwithin{lemma}{section}
\counterwithin{theorem}{section}
\counterwithin{proposition}{section}
\counterwithin{corollary}{section}
\counterwithin{definition}{section}
\counterwithin{equation}{section}
\section{Proof of Theorem~\ref{MART_THM_Vector}}
Let us first state a result which follows from the arguments of Chapter~4 of~\cite{vershynin2018high}. See also Proposition~8.1 of~\cite{sarkar2018fast}.
\begin{proposition}[\cite{vershynin2018high}]\label{MART_PROP_Covering}
	Consider a matrix $M\in\REAL^{m\times n}$. Let  $S^{n-1}$ denote the unit sphere. Then for any $\epsilon>0$:
	\begin{align*}
	\mathbb{P}\paren{\norm{M}_2\ge t}\le \paren{1+\frac{2}{\epsilon}}^n \max_{x\in S^{n-1}}\mathbb{P}\paren{\norm{Mx}_2\ge t\paren{1-\epsilon}}
	\end{align*}
	\hfill $\diamond$
\end{proposition}

Second, we state Theorem~1 of~\cite{abbasi2011improved}, which upper bounds self-normalized martingales.
 \begin{theorem}[Theorem 1 in~\cite{abbasi2011improved}]\label{MART_THM_Abbasi}
 	Let $\set{\F_t}_{t=0}^{\infty}$ be a filtration. Let $\eta_{t}\in\REAL$, $t\ge 0$ be  real-valued, $\F_t$-measurable, and conditionally $1-$sub-Gaussian:
 	\begin{equation}\label{AUX_EQN_subgaussian}
 	\expe{e^{\lambda\eta_{t}}|\F_{t-1}}\le e^{\lambda^2/2},\text{ for all }\lambda\in \REAL.
 	\end{equation}
 	Let $X_{t}\in\REAL^{d}$, $t\ge 0$ be vector valued and $\F_{t-1}-$measurable. Assume that $V$ is a $d\times d$ positive definite matrix. For any $t\ge 0$, define:
 	\[
 	\bar{V}_t=V+\sum_{s=1}^{t}X_sX_s^*,\qquad S_t=\sum_{s=1}^{t}\eta_s X_s.
 	\]
 	Then, for any $\delta>0$, with probability at least $1-\delta$, for all $t\ge 0$
 	\begin{equation}
 	\norm{\bar{V}_t^{-1/2} S_t }^2_2\le 2 \log\frac{\det\paren{\bar{V}_t}^{1/2}\det\paren{V}^{-1/2}}{\delta}
 	\end{equation}
 	\hfill $\diamond$
 \end{theorem}

Finally, we state a standard linear algebra result. We include the proof for completeness.
\begin{lemma}[Block Matrix Norm]\label{MART_LEM_Block_Matrices}
	Assume
	\[
	M=\matr{{cccc}M_1&M_2&\dots&M_r},
	\]	
	for matrices of appropriate dimensions. 
	Then:
	\[
	\norm{M}_2\le 
	\sqrt{r}\max_{i=1,\dots,r}\norm{M_i}_2
	\]
		\hfill $\diamond$
\end{lemma}
\begin{proof}
	Consider a vector $x$ such that $Mx$ is defined. Then, from triangle inequality, the definition of matrix norm and Cauchy-Schwartz: 
	\begin{align*}
	\norm{Mx}_2=\norm{M_1x_1+\dots+ M_rx_r}_2&\le \norm{M_1}_2 \norm{x_1}_2+\dots+\norm{M_r}_2 \norm{x_r}_2\\
	&\le \sqrt{\norm{M_1}^2_2+\dots+\norm{M_r}^2_2}\norm{x}_2 \\
	&\le \sqrt{r}\max_{i=1,\dots,r}\norm{M_i}_2 \norm{x}_2,
	\end{align*}
	where we used that $\norm{x}^2_2=\norm{x_1}_2^2+\dots+\norm{x_r}^2$.
\end{proof}

Now, we can prove Theorem~\ref{MART_THM_Vector}.
Notice that:
\[
S_t=\matr{{ccc}\displaystyle\sum_{s=1}^{t}X_s\eta^*_{s}&\dots&\displaystyle\sum_{s=1}^{t}X_{s}\eta^*_{s+r-1}}.
\]
We can analyze each component $\bar{V}_t^{-1/2}\sum_{s=1}^{t}X_s\eta^*_{k+i}$ separately and apply a union bound afterwards, since by Lemma~\ref{MART_LEM_Block_Matrices}:
\begin{equation}\label{MART_EQN_Norm_Bound}
\norm{\bar{V}_t^{-1/2}S_t}_{2}\le \sqrt{r}\max_{i=0,\dots,r-1}\norm{\bar{V}_t^{-1/2}\sum_{s=1}^{t}X_s\eta^*_{s+i}}_2.
\end{equation}

Now, fix a $0\le i<r$  and let
\[S^{i}_{t}\triangleq\sum_{s=1}^{t}X_{s}\eta^*_{s+i}.\]
Consider an arbitrary element of the unit sphere $\xi\in S^{m-1}$.
The scalar $\eta^*_{s+i}\xi$ is conditionally $1$-sub-Gaussian and satisfies the conditions of Theorem~\ref{MART_THM_Abbasi}.
Thus, with probability at least $1-\frac{\delta}{r5^m}$:
\begin{align*}
\norm{\bar{V}_t^{-1/2}S^{i}_t\xi }^2_2\le \C_{XH}
\triangleq  2\paren{\log\frac{r5^m}{\delta}+\frac{1}{2}\log\det\bar{V}_tV^{-1}} 
\end{align*}
Now, we apply Proposition~\ref{MART_PROP_Covering} for $\epsilon=1/2$:
\begin{align*}
&\mathbb{P}\paren{\norm{\bar{V}_t^{-1/2}S^{i}_t}_2\ge  2\sqrt{\C_{XH}}} \\
&\le	5^m \max_{\xi \in S^{m-1}}\mathbb{P}\paren{\norm{\bar{V}_t^{-1/2}S^{i}_{t}\xi}_2\ge  \sqrt{\C_{XH}}}\le \frac{\delta}{r}.
\end{align*}
Finally, by~\eqref{MART_EQN_Norm_Bound} and a union bound over all components:
\begin{align*}
\P\paren{\norm{\bar{V}_t^{-1/2}S_t}_{2}\ge 2\sqrt{r}\sqrt{\C_{XH}}}&\le \P\paren{\max_{i=0,\dots,r-1}\norm{\bar{V}_t^{-1/2}S^i_t}_{2}\ge 2\sqrt{\C_{XH}}}\\
&\le \sum_{i=0}^{r-1}\mathbb{P}\paren{\norm{\bar{V}_t^{-1/2}S^i_t}_{2}\ge 2\sqrt{\C_{XH}}}\le
\delta. 
\end{align*}
\hfill $\qed$

 \section{Persistence of Excitation}
The main focus of this section is the proof of Theorem~\ref{MAIN_THM_PE}, which provides finite sample guarantees for the PE of the past noises and the past outputs. We also include upper bounds for the sample correlations of the past outputs and the states $\hat{x}_k$. 
Finally, we provide the definition of $N_0,N_1$ that we hided in the main theorem statements.

The following result shows that with high probability, the past noises are persistently exciting. It follows from Lemma~C.2 of~\cite{oymak2018non}, which in turn is based on results for random circulant matrices~\cite{krahmer2014suprema}.
\begin{lemma}[Noise PE]\label{POE_LEM_Past_Noises}
	Consider the conditions of Theorem~\ref{MAIN_THM_PE} and the definition of $\delta_{N}$ in~\eqref{MAIN_EQN_DELTA_N}:
			\begin{equation*}
	\delta_N\triangleq\paren{2(N+p-1)m}^{-\log^2\paren{2pm}\log\paren{2(N+p-1)m}}.
	\end{equation*} 
	There exists a universal constant $c$ (independent of system and algorithm parameters) such that if $N\ge 2cpm\log{1/\delta_N}$, then with probability at least $1-\delta_N$ the event:
	\[
	\E_{E}=\set{\frac{1}{2}\Sigma_{E}\preceq \frac{1}{N}\Tm_p E_-E^*_-\Tm^*_p },
	\]
	occurs, where
	$\Sigma_{E}$ is defined in~\eqref{MAIN_EQN_Past_Noise_Covariance}.
	\hfill $\diamond$
\end{lemma}
\begin{proof}
	We can rewrite $E^{-}_k=\diag{\bar{R}^{1/2},\dots,\bar{R}^{1/2}}U^{-}_k$, where $U^{-}_k$ is defined similarly to $E^{-}_k$ but has components with unit covariance.
	Now from Lemma C.2 of~\cite{oymak2018non} applied on $U^{-}_k$ we obtain that with probability at least $1-\delta_N$:
	\[
	\frac{1}{N}	\sum_{k=p}^{N+p-1} U^{-}_k(U^{-}_k)^*\succeq I_{mp}/2.
	\]
	Multiplying by $\Tm_p \diag{\bar{R}^{1/2},\dots}$ from the left and $\diag{\bar{R}^{1/2},\dots}\Tm_p^*$ from the right gives the desired result.	
\end{proof}

From the above lemma it follows that $N$ should be large enough to guarantee PE for the noises. In particular, $N$ should be larger than $N_0$, where
\begin{align}\label{POE_EQN_N0}
N_0&=\min\set{N:\: N\ge 2cpm\log{1/\delta_N}}\\
&=\min\set{N:\: N\ge 2cpm\log^2\paren{2pm}\log^2\paren{2(N+p-1)m}}.\nonumber
\end{align}
Such a $N_0$ exists since the term $2cpm\log^2\paren{2pm}\log^2\paren{2(N+p-1)m}$ depends logarithmically on $N$.

To guarantee PE for the noises, we also need to show that the smallest singular value $\sigma_E=\sigma_{\min}\paren{\Sigma_E}$ is positive. 
\begin{lemma}\label{POE_LEM_Sigma_E}
	Let $\Sigma_E$ be as in~\eqref{MAIN_EQN_Past_Noise_Covariance}. Then:
	\[
	\sigma_{E}\ge \sigma_{\min}\paren{R}>0.
	\]
	\hfill $\diamond$
\end{lemma}
\begin{proof}
	The past errors $\Tm_p E^{-}_k$ can be rewritten as:
	\begin{equation*}
	\Tm_p E^{-}_k=Y^{-}_k-\O_p\hat{x}_{k-p}=\O_p\paren{x_{k-p}-\hat{x}_{k-p}}+V^{-}_k+\Tm W^{-}_k,
	\end{equation*}
	where $V^{-}_{k}$, $W^{-}_k$ are defined similarly to $Y^{-}_k$ and consist of the past measurement and process noises respectively. Matrix $\Tm$ is a block Toeplitz matrix; we omit its analytical expression. By independence of the measurement noise, the process noise and $x_{k-p}-\hat{x}_{k-p}$:
	\begin{equation*}
	\Sigma_{E}\succeq \expe{V^{-}_{k}\paren{V^{-}_{k}}^*}=\diag{R,\dots,R}.
	\end{equation*}
	This in turn implies $\sigma_E\ge \sigma_{\min}\paren{R}$.
	By Assumption~\ref{ASS_Kalman}, $R\succ 0$ and $\sigma_{\min}\paren{R}>0$.
\end{proof}

\begin{lemma}[Markov upper bounds]\label{POE_LEM_upper_bounds}
	Consider system~\eqref{EQN_System_Innovation} and recall the definition of $\Gamma_k$, $\Sigma_E$ in~\eqref{EQN_State_Prediction_Covariance},~\eqref{MAIN_EQN_Past_Noise_Covariance}. We have the following upper bounds:
	\begin{align}\label{POE_EQN_state_covariance_upper}
	&	\mathbb{P}\paren{\norm{\hat{X}\hat{X}^*}_2\ge {N\frac{\Tr{\Gamma_{N-1}}}{\delta}}}\le \delta\\ \label{POE_EQN_output_covariance_upper}
	&	\mathbb{P}\paren{\norm{Y_{-}Y^*_{-}}_2\ge N\frac{\norm{\O_p}^2_2 \Tr \Gamma_{N-1}+\Tr\Sigma_{E}}{\delta}}\le \delta.
	\end{align}
\end{lemma}
\begin{proof}
	We only show~\eqref{POE_EQN_output_covariance_upper}. The proof of~\eqref{POE_EQN_state_covariance_upper} is similar.
	From Markov's inequality we obtain:
	\[
	\mathbb{P}\paren{\norm{Y_{-}Y^*_{-}}_2\ge \epsilon}\le \frac{\mathbb{E}{\norm{Y_{-}Y^*_{-}}_2}}{\epsilon}.
	\]	
	What remains is to bound the expectation in the right-hand side.
	Notice that $\snorm{Y^{-}_k\paren{Y^{-}_k}^*}_2= \Tr Y^{-}_k\paren{Y^{-}_k}^*$, since $Y^{-}_{k}$ has unit rank. Hence from the triangle inequality:
	\[
	\mathbb{E}{\norm{Y_{-}Y^*_{-}}_2}\le \sum_{k=p}^{N+p-1} \Tr{\expe{{Y^{-}_{k}\paren{Y^{-}_k}^*}}}.
	\]
	The right hand side of the above inequality is
	\[
	\Tr{\expe{{Y^{-}_{k}\paren{Y^{-}_k}^*}}}= \Tr\paren{\O_p\Gamma_{k-p} \O^*_p+\Sigma_E}
	\]
	The result follows from
	 $\Tr\paren{\O_p\Gamma_{k-p} \O^*_p}\le \norm{\O_p}^2_2\Tr \Gamma_{k-p}$
	 along with 
	 \[
	 \Tr \sum_{k=p}^{N+p-1}\Gamma_{k-p}\le N \Tr \Gamma_{N-1}
	 \]
	  since the sequence $\Gamma_{k}$ is monotone (see the following Lemma).
\end{proof}
The following result is standard and we include it for completeness.
\begin{lemma}[Monotonicity of $\Gamma_k$]
Consider system~\eqref{EQN_System_Innovation}, under Assumption~\ref{ASS_Kalman}. The sequence $\Gamma_k=\expe{\hat{x}_k\hat{x}^*_k}$ is monotone: $\Gamma_{k}\succeq \Gamma_{k-1}$.
\end{lemma}
\begin{proof}
	Notice that since $\hat{x}_0=0$, we have $\Gamma_0=0$. Define $\bar{Q}=K\bar{R}K^*$. By the orthogonality principle, $\hat{x}_k$ and $e_k$ are uncorrelated. Hence
	\[
	\Gamma_{k}=\mathcal{L}\paren{\Gamma_{k-1}}\triangleq A\Gamma_{k-1}A^*+\bar{Q}.
	\]
For $k=1$ we obtain:
	\[
	\Gamma_{1}=	\bar{Q}\succeq 0=\Gamma_{0}.
	\]
	But the operator $\mathcal{L}$ is monotone, which implies:
	\[
	\Gamma_{2}=	\mathcal{L}\paren{\Gamma_{1}} \succeq \mathcal{L}\paren{\Gamma_{0}}=\Gamma_{1}
	\]
	The result
$
	\Gamma_{k}\succeq \Gamma_{k-1}
$
	 follows by induction.
\end{proof}

\subsection*{Proof of Theorem~\ref{MAIN_THM_PE}}
Some arguments are similar to Section~9 of~\cite{sarkar2018fast}.

\textbf{Step 1: Noise PE.} 
Under the condition $N\ge N_0$, from Lemma~\ref{POE_LEM_Past_Noises} the event $\E_E$ occurs with probability at least $1-\delta_N$. 

\textbf{Step 2: Cross terms are small.} 
Next, we show that the
cross terms $\hat{X}E^*_{-}$ are small. As in Lemma~\ref{POE_LEM_Past_Noises}, express $E_{-}=\diag{\bar{R}^{1/2},\dots,\bar{R}^{1/2}}H_-$, where $H_-$ is defined similarly to $E_-$ but has unit variance components.
Define:
\[
\bar{V}_N=\hat{X}\hat{X}^*+\frac{N}{\norm{\O_p}^2_2}I_{n},\quad V=\frac{N}{\norm{\O_p}^2_2}I_{n},\quad V_N=\hat{X}\hat{X}^* ,\quad S_N=\hat{X}E^*_{-}.
\]
Notice that
\[
\norm{\bar{V}_N^{-1/2} S_N }^2_2\le \norm{\bar{R}}_2	\norm{\bar{V}_N^{-1/2}\hat{X}H^*_{-}}^2_2
\]
Hence, by Theorem~\ref{MART_THM_Vector} applied to $\bar{V}_N$, $\hat{X}H^*_{-}$ the event:
\begin{align*}
\E_1=\set{	\norm{\bar{V}_N^{-1/2} S_N }^2_2 \le 8p\norm{\bar{R}}_2\paren{\log\frac{p5^m}{\delta}+\frac{1}{2}\log\det\bar{V}_NV^{-1}}}
\end{align*}
occurs with probability at least $1-\delta$.

Next we upper bound term $\bar{V}_N$. From Lemma~\ref{POE_LEM_upper_bounds}, the event:
\[
\E_2=\set{V_N\preceq N\frac{\Tr \Gamma_{N-1}}{\delta}I_n}
\]
occurs with probability at least $1-\delta$. This implies that:
\begin{align*}
\log\det\bar{V}_NV^{-1} &\le \log\det\clint{ \paren{\frac{N\Tr \Gamma_{N-1}}{\delta}+\frac{N}{\norm{\O_p}^2_2}}I_n\frac{\norm{\O_p}^2_2}{N}}\\
&=\log\paren{\frac{\norm{\O_p}^2_2\Gamma_{N-1}}{\delta}+1}^{n}=n\log\paren{\frac{\norm{\O_p}^2_2\Gamma_{N-1}}{\delta}+1}
\end{align*}

Combining the two events $\E_1,\E_2$ and by a union bound, with probability at least
$1-2\delta$ the event:
\[
\E_{XE}=\set{
	\norm{\bar{V}_N^{-1/2} S_N }^2_2\le\C_{XE}\norm{\bar{R}}_2},
\]
occurs, where
\[
\C_{XE}\triangleq 8p  \paren{\frac{n}{2}\log\paren{\frac{\norm{\O_p}^2\Tr \Gamma_{N-1}}{\delta}+1}+\log \frac{p5^m}{\delta}}.
\]
As a consequence, if $u\in\REAL^{mp}$, $\norm{u}_2=1$ is an arbitrary unit vector:
\begin{align}
&|u^*\O_p \hat{X}E^*_-\Tm^*_pu| \le \snorm{u^*\O_p \bar{V}^{1/2}_N  \bar{V}^{-1/2}_NS_N\Tm^*_p}_2\nonumber\\
&\le \sqrt{u^*\O_p\hat{X}\hat{X}^*\O^*_pu+N\frac{u^*\O_p\O^*_pu}{\norm{\O_p}^2_2}}\sqrt{\C_{XE}\norm{\bar{R}}_2}\norm{\Tm_p}_2 \nonumber\\
&\le \sqrt{u^*\O_p\hat{X}\hat{X}^*\O^*_pu+N}\sqrt{\C_{XE}\norm{\bar{R}}_2}\norm{\Tm_p}_2,\text{ conditioned on }\E_{XE} \label{POE_EQN_Cross_AUX}
\end{align}

\textbf{Step 3: Output PE} 
Consider an arbitrary unit vector $u\in\REAL^{mp}$, $\norm{u}_2=1$. Consider the events $\E_{E}$ and $\E_{XE}$ from steps 1,2. With probability $1-\delta_N-2\delta$, since $N\ge N_0$ the event $\E_{E}\cap \E_{XE}$ occurs. It remains to show that on $\E_{E}\cap \E_{XE}$ the outputs satisfy PE for sufficiently large $N$.
Define
\[
\alpha\triangleq \frac{1}{N} u^*\O_p\hat{X}\hat{X}^*\O^*_pu,\, \beta\triangleq \frac{1}{N} u^*\Tm_p E_- E^*_-\Tm^*_p u
\] 
From~\eqref{MAIN_EQN_Output_Correlations_Identity},~\eqref{POE_EQN_Cross_AUX} for $N\ge N_0$ on $\E_{E}\cap \E_{XE}$:
\[
\frac{1}{N}u^*Y_-Y^*_-u\ge \alpha+\beta-\underbrace{2\norm{\Tm_p}_2\sqrt{\frac{\C_{XE}\norm{\bar{R}}_2}{N}}}_{\gamma_N}\sqrt{\alpha+1}
\]
with $\beta\ge {\sigma_E}/{2}$.
Now let $N_1$ be such that:
\begin{equation}\label{POE_EQN_N1}
N_1=\min\set{N:\: \gamma_N\le \min\set{1,\frac{\sigma_{E}}{4}}}.
\end{equation}
Since $\C_{XE}$ grows at most logarithmically with $N$, $N_1$ always exists.
Now, since $N\ge N_1$ and $\beta\ge \sigma_E/2$:
\begin{equation*}
\alpha+\beta-\gamma_N\sqrt{\alpha+1}\ge \frac{\alpha+\beta}{2}.
\end{equation*}
The above inequality follows from the following lemma.
\hfill $\qed$

\begin{lemma}[Minimum of function]\label{POE_LEM_Minimum_Function}
	Let $\beta\ge b>0$, for some $b>0$ and consider the function:
	\[
	f(\alpha,\beta)=\frac{\alpha+\beta}{2}-\gamma\sqrt{\alpha+1},\text{ for }\alpha\ge 0,\,\beta\ge b>0
	\]
	If $\gamma \le 1,\frac{b}{2}$, then 
	\[f(\alpha,\gamma)\ge 0,\,\text{for all }\alpha\ge 0,\,\beta\ge b>0.\]
	\hfill $\diamond$
\end{lemma}
\begin{proof}
	By elementary calculus:
	\begin{align*}
	\min_{\alpha\ge 0 } f(\alpha,\beta)=\begin{array}{cc} \frac{\beta-1-\gamma^2}{2},&\text{ if }\gamma\ge 1\\\frac{\beta}{2}-\gamma,&\text{ if }\gamma<1 \end{array}
	\end{align*}
	Thus, if $\gamma \le 1,b/2$, we have $f\paren{\alpha,\beta}\ge 0$.
\end{proof}

\section{Proof of Theorem~\ref{MAIN_THM_Hankel_Matrix}}

\textbf{Step 1:}
Since $N\ge N_0,N_1$, from Theorem~\ref{MAIN_THM_PE}, with probability at least $1-\delta_N-2\delta$, the event $\E_Y\cap\E_E$ occurs.

\textbf{Step 2:}
Next, we analyze the cross term. Define:
\[
Z_N=Y_-Y^*_-,\quad\bar{Z}_N=Z_N+N\frac{\sigma_E}{4}I_{mp},\quad S_N=\Tm_fE_+Y^*_-
\] 
Notice that on the event $\E_Y\cap\E_E$, we have that $\bar{Z}_N\preceq 2 Z_N$ since $Z_N\succeq N\frac{\sigma_E}{4}I_{mp}$. Hence,
\[
\norm{S_NZ_N^{-1/2}}_2\le \sqrt{2} \norm{S_N\bar{Z}_N^{-1/2}}_2.
\]
Now the proof continues as in the case of cross-terms in the proof of Theorem~\ref{MAIN_THM_PE}. 
We apply Theorem~\ref{MART_THM_Vector} to $\bar{Z}_N, S_N$ and use Lemma~\ref{POE_LEM_upper_bounds} to upper bound $\bar{Z}_N$. Then, conditioned on $\E_Y\cap\E_E$, with probability $1-2\delta$:
\begin{equation}\label{Proof_EQN_True_Bound}
\norm{S_NZ_N^{-1/2}}^2_2\le 16\norm{\bar{R}}_2\norm{\Tm_f}^2_2 \paren{\frac{fmp}{2}\log\frac{\kappa_{N}}{\delta}+f\log\frac{5^mf}{\delta}},
\end{equation}
where
\[
\kappa_{N}=\frac{4}{\sigma_{E}}\paren{\norm{\O_p}^2_2\Tr\Gamma_{N-1}+\Tr\Sigma_{E}}+\delta.
\]
Next, we bound $\norm{Z^{-1/2}_N}$ separately on the event $\E_Y\cap \E_E$ by $\frac{2}{\sqrt{N\sigma_E}}$.

Finally, conditioned on $\E_E\cap \E_Y$, with probability at least $1-2\delta$:
\begin{equation}\label{Proof_EQN_Cross_Term_Final}
\norm{\T_fE_+Y^*_-\paren{Y_-Y^*_-}^{-1}}\le \frac{\C_1}{\sqrt{N}} \sqrt{\frac{fmp}{2}\log\frac{\kappa_{N}}{\delta}+f\log\frac{5^mf}{\delta}}.
\end{equation}

\textbf{Step 3:}
We bound the Kalman truncation term. Recall that:
\begin{align*}
\hat{X}Y^*_-(Y_-Y^*_-)^{-1}=\O_p^{\dagger}&\paren{I_{mp}-\Tm_pE_{-}E_{-}^*\Tm^*_p(Y_-Y^*_-)^{-1}\right.\nonumber\\&\left.-\Tm_pE_{-}\hat{X}^*\O_p^*(Y_-Y^*_-)^{-1}}
\end{align*}
On the event $\E_E\cap \E_Y$, we have:
\[
\norm{\Tm_pE_{-}E_{-}^*\Tm^*_p(Y_-Y^*_-)^{-1}}_2 \le 2.
\]
since $Y_-Y^*_-\succeq \frac{1}{2}\Tm_pE_{-}E_{-}^*\Tm^*_p$.
Hence we obtain:
\[
\norm{\hat{X}Y^*_-(Y_-Y^*_-)^{-1}}_2\le \norm{\O_p^{\dagger}}_2\paren{3+\norm{\Tm_p}_2\norm{E_{-}\hat{X}^*\O_p^*(Y_-Y^*_-)^{-1}}_2}
\]
From the discussion in Section~\ref{Section_Analysis_Hankel}, we only need to find $N_2$ such that for $N\ge N_2$ with high probability:
\[
\norm{\Tm_p}_2\norm{E_{-}\hat{X}^*\O_p^*Z_N^{-1}}_2\le 1.
\]

Define
\[
B_N=\O_p\hat{X}\hat{X}\O_p,\,B=\frac{\sigma_E}{2}NI_{mp},\,\bar{B}_N=B_N+B.
\]
Notice that on the event $\E_E\cap\E_Y$, we have $Y_-Y^*_-\succeq \frac{1}{2}\bar{B}_N$, which implies:
\[
\snorm{E_{-}\hat{X}^*\O_p^*Z_N^{-1/2}}\le \sqrt{2} \snorm{E_{-}\hat{X}^*\O_p^*\bar{B}_N^{-1/2}}
\]
Now we can treat the right-hand side in the same way as the cross-term above. By an application of Theorem~\ref{MART_THM_Vector} and Lemma~\ref{POE_LEM_upper_bounds}, we obtain that conditioned on $\E_Y\cap\E_E$, with probability $1-2\delta$:
\begin{align*}
\norm{E_{-}\hat{X}^*\O_p^*Z_N^{-1}}&\le \sqrt{2}\norm{E_{-}\hat{X}^*\O_p^*\bar{B}_N^{-1/2}}\norm{Z_N^{-1/2}}_2\\
&\le 8\sqrt{\frac{\snorm{\bar{R}}_2}{\sigma_E}}\frac{\C_{N}}{\sqrt{N}},
\end{align*}
where 
\[
\C_{N}=\sqrt{\frac{mp^2}{2}\log\paren{\frac{2\norm{\O_p}^2_2\Gamma_{N-1}}{ \delta \sigma_E}+1}+p\log\frac{p5^m}{\delta}}
\]
Thus, we define:
\begin{equation}\label{TRUNC_EQN_N2}
N_2=\min\set{N:\:8\sqrt{\frac{\snorm{\bar{R}}_2}{\sigma_E}}\snorm{\Tm_p}_2\frac{\C_{N}}{\sqrt{N}}\le 1}.
\end{equation}
Such an $N_2$ exists since $\C_N$ grows at most logarithmically with $N$.

\textbf{Step 4: Final expression} From the previous step and a union bound, for $N\ge N_0,N_1,N_2$ with probability at least $1-\delta_N-6\delta$:
\begin{equation}\label{Proof_EQN_Basic_Result_Equivalent}
\norm{G-\hat{G}}_2\le \frac{\C_1}{\sqrt{N}} \sqrt{\frac{fmp}{2}\log\frac{\kappa_{N}}{\delta}+f\log\frac{5^mf}{\delta}}+\mathcal{C}_2\norm{A-KC}_2^p,
\end{equation}

\textbf{Step 5: Simplification of final expression}
To simplify the final expression, we use
\[
\frac{fmp}{2}\log \kappa_{N}+f\log\paren{5^mf} \le fmp\paren{\log \kappa_{N}+\log\paren{5f}}=fmp \log\paren{ 5f\kappa_{N}}
\]
and
\[
\frac{fmp}{2}\log\frac{1}{\delta}+f\log\frac{1}{\delta}\le fmp \log\frac{1}{\delta}
\]
 since $p\ge n+1\ge 2$.

\section{Proof of Theorem~\ref{SVD_THM_Main}}
The proof follows the steps of~\cite{oymak2018non}.
\subsection*{Step 1: Bounds for observability/controllability matrix}
Denote the rank $n$ approximation of $\hat{G}$ by
\[
\hat{G}_{n}\triangleq\hat{\O}_f\hat{\K}_p
\]
By definition, $\hat{G}_n$ is the matrix which minimizes $\norm{\hat{G}-M}_2$, among all rank $n$ matrices $M$. Thus, by optimality:
\[
\norm{\hat{G}-\hat{G}_n}_2\le \norm{\hat{G}-G}_2,
\]
since $G$ has also rank $n$.
As a result, we have:
\begin{equation}\label{ROB_EQN_Bound_Rank_n_approx}
\norm{G-\hat{G}_n}_2\le \norm{G-\hat{G}}_2+\norm{\hat{G}-\hat{G}_n}_2 \le 2\norm{\hat{G}-G}_2
\end{equation}

From~\eqref{ROB_EQN_Bound_Rank_n_approx} and the robustness condition~\eqref{SVD_EQN_robustness_condition} we have:
\begin{equation}\label{ROB_EQN_Condition_Procrustes}
\norm{G-\hat{G}_n}_2 \le 2\norm{\hat{G}-G}_2\le \frac{\sigma_n\paren{G}}{2}.
\end{equation}
Hence, we can now apply Theorem 5.14~of~\cite{tu2016low}, which states that
 there exists an orthonormal matrix $T$ such that:
\begin{equation}\label{ROB_EQN_Application_Procrustes}
\sqrt{\norm{\hat{\O}_f-\bar{\O}_fT}^2_F+\norm{\hat{\K}_p-T^*\bar{\K}_p}^2_F}\le \sqrt{\frac{2}{\paren{\sqrt{2}-1}\sigma_n\paren{G}}}\norm{G-\hat{G}_n}_F,
\end{equation}
where $\norm{\cdot}_F$ denotes the Frobenius norm.

Since matrices $G,\hat{G}_n$ have rank-$n$, the sum $G-\hat{G}_n$ has rank at most $2n$. Thus, we can bound the Frobenius norm $\norm{G-\hat{G}_n}_F$ in terms of spectral norm:
\begin{equation}\label{ROB_EQN_Frobenius_Bound}
\norm{G-\hat{G}_n}_F \le \sqrt{2n}\norm{G-\hat{G}_n}_2\le 2 \sqrt{2n} \norm{\hat{G}-G}_2
\end{equation}
where the second inequality follows from~\eqref{ROB_EQN_Bound_Rank_n_approx}.
For simplicity, we also use $\frac{2}{\sqrt{2}-1}\le 5$.
Thus, from~\eqref{ROB_EQN_Application_Procrustes} and the above inequalities:
\[
\sqrt{\norm{\hat{\O}_f-\bar{\O}_fT}^2_F+\norm{\hat{\K}_p-T^*\bar{\K}_p}^2_F}\le 2\sqrt{\frac{10n}{\sigma_n\paren{G}}}\norm{G-\hat{G}}_2.
\]
Finally, since the spectral norm is always smaller than the Frobenius one:
\begin{equation}\label{ROB_EQN_Observability}
\sqrt{\norm{\hat{\O}_f-\bar{\O}_fT}^2_2+\norm{\hat{\K}_p-T^*\bar{\K}_p}^2_2}\le 2\sqrt{\frac{10n}{\sigma_n\paren{G}}}\norm{G-\hat{G}}_2,
\end{equation}
As a corollary,
\[
\max\set{\norm{\hat{\O}_f-\bar{\O}_fT}_2,\norm{\hat{\K}_p-T^*\bar{\K}_p}_2}\le 2\sqrt{\frac{10n}{\sigma_n\paren{G}}}\norm{G-\hat{G}}_2,
\]

\subsection*{Step 2: Bounds for system parameters}

\textbf{Bounds for $C,K$}

Since $\hat{C}-\bar{C}T$ is a sub-matrix of $\hat{\O}_f-\bar{\O}_fT$, we immediately obtain:
\[
\norm{\hat{C}-\bar{C}T}_2 \le \norm{\hat{\O}_f-\bar{\O}_fT}_2\le 2\sqrt{\frac{10n}{\sigma_n\paren{G}}}\norm{G-\hat{G}}_2.
\] 
Similarly, for $K$ we have:
\[
\norm{\hat{K}-T^*\bar{K}}_2 \le \norm{\hat{\K}_p-T^*\bar{\K}_p}_2\le 2\sqrt{\frac{10n}{\sigma_n\paren{G}}}\norm{G-\hat{G}}_2
\] 
\textbf{Bounds for $A$}

For simplicity, denote $\hat{M}\triangleq \hat{\O}^h_f$, $\bar{M}\triangleq \bar{\O}^h_f$ and $\hat{N}\triangleq \hat{\O}^l_f$, $\bar{N}\triangleq \bar{\O}^l_f$. Based on this notation:
\[
\hat{A}=\hat{M}^{\dagger}\hat{N},\quad \bar{A}=\bar{M}^{\dagger}\bar{N}.
\]
After some algebraic manipulations:
\begin{equation*}
\hat{A}-T^*\bar{A}T=\paren{\hat{M}^{\dagger}-T^*\bar{M}^{\dagger}}\bar{N}T+\hat{M}^{\dagger}\paren{\hat{N}-\bar{N}T}.
\end{equation*}
First, notice that $\norm{\bar{N}}_2\le \norm{\bar{\O}_f}_2=\sqrt{\norm{G}}$, where the inequality follows from the fact that $\bar{N}$ is a submatrix of $\bar{\O}_f$; equality follows from the definition of $\bar{\O}_f=U_1\Sigma_1^{1/2}$. Second, $\norm{\hat{M}^{\dagger}}_2=\frac{1}{\sigma_n\paren{\hat{M}}}\le \frac{1}{\sigma_o}$ and third 
\[
\norm{\hat{N}-\bar{N}T}_2\le \norm{\hat{\O}_f-\bar{\O}_fT}_2,
\]
since $\hat{N}-\bar{N}T$ is a submatrix of $\hat{\O}_f-\bar{\O}_fT$.
Finally, from Theorem~4.1 of~\cite{wedin1972perturbation}
\[
\norm{\hat{M}^{\dagger}-T^*\bar{M}^{\dagger}}_2\le \norm{\hat{M}-T^*\bar{M}}_2 \max\set{\frac{1}{\sigma^2_n\paren{\hat{M}}},\frac{1}{\sigma^2_n\paren{\bar{M}}}}\le \norm{\hat{\O}_f-\bar{\O}_fT}_2 \frac{1}{\sigma^2_o}.
\]
Combining all previous bounds, we obtain
\begin{equation*}
\norm{\hat{A}-T^*\bar{A}T}_2\le \paren{\frac{\sqrt{G}}{\sigma^2_o}+\frac{1}{\sigma_o}}\norm{\hat{\O}_f-\bar{\O}_fT}_2.
\end{equation*}

\section{Bounds for system matrices}
In this section, we formally prove bounds for $\O_k,\Tm_k,\Gamma_k$, which we implicitly used in Section~\ref{Section_Analysis_Hankel}.

First, we prove a standard result for the norm of (block) Toeplitz matrices.
\begin{lemma}[Toeplitz norm]\label{ALG_LEM_TOEPLITZ}
	Let $M\in \REAL^{m_1n\times m_2n}$, for some integers $n,m_1,m_2$ be an (upper) block triangular Toeplitz matrix:
	\[
	M=\matr{{cccccc}M_1&M_2&M_3&\cdots&\cdots&M_n\\0&M_1&M_2& & &M_{n-1}\\\vdots & & \ddots & \ddots& &\vdots\\ \\ \\ & & & &M_1&M_2\\0&0&\cdots& &0&M_1},
	\]
	where $M_{i}\in \REAL^{m_1\times m_2}$, $i=1,\dots,n$.
	Then:
	\[
	\norm{M}_2\le \sum_{i=1}^{n}\norm{M_i}_2
	\]
\end{lemma}
\begin{proof}
	A standard technique in the analysis of Toeplitz matrices is to write them in terms of linear combinations of powers of companion matrices (see~\cite{horn2012matrix}, equation (0.9.7)).
	We can write $M$ as:
	\[
	M=I_n\otimes M_1+\sum_{i=1}^{n-1}J_n^{i}\otimes M_i
	\]
	where $\otimes$ is the Kronecker product, $I_n$ is the identity matrix of dimension $n$, and $J_n$ is the companion matrix:
	\[
	J_n=
	J=\matr{{ccccc}0&1&\cdots&0&0\\0&0&\ddots &\vdots &\vdots\\& & \ddots \\ & & &0&1\\0&0&\cdots &0&0},
	\]
	But we have $\norm{D\otimes F}_2=\norm{D}_2\norm{F}_2$ for any matrices $D,F$ (see Theorem 4.2.15 in~\cite{horn1994topics}). Also the companion matrix has norm one $\norm{J_n}_2=1$ since $J_nJ_n^*=\diag{1,\dots,1,0}$. The result follows from the triangle inequality.
\end{proof}

Next, we provide a bound for the sum of powers of $A$.
\begin{lemma}
Consider the series $S_t=\sum_{i=0}^{t}\norm{A^{i}}_2$. We have the following two cases:
\begin{itemize}
		\item If the system is asymptotically stable $\rho(A)<1$, then $\norm{S_t}_2=\O\paren{1}$
	\item If the system is marginally stable ($\rho(A)= 1$), then $\norm{S_t}_2=\O\paren{t^{\kappa}}$, where $\kappa$ is the largest Jordan block of $A$ corresponding to a unit circle eigenvalue $\abs{\lambda}=1$.
\end{itemize}
\end{lemma}
\begin{proof}
	
	\textbf{Proof of first part.}
	By Gelfand's formula~\cite{horn2012matrix}, for every $\epsilon>0$, there exists a $i_0=i_0(\epsilon)$ such that $\norm{A^i}\le (\rho(A)+\epsilon)^i$, for all $i\ge i_0$.
	Just pick $\epsilon$ such that $\rho\paren{A}+\epsilon<1$. Then,
	\[
	S_t\le \sum_{i=0}^{i_0} \norm{A^i}_2+\frac{1}{1-\rho(A)-\epsilon}=\O\paren{1}.
	\] 
	
\textbf{Proof of second part.}
Assume that $A$ is equal to a $n\times n$ Jordan block corresponding to $\lambda=1$. The proof for the other cases is similar.
Then we have that:
\[
A^{i}=\matr{{cccc}1& {i }\choose {1}&\dots& {i} \choose{ n-1}\\0&1&\dots&{i} \choose{ n-2}\\ & &\ddots &\\0&0&\dots&1}
\]
By Lemma~\ref{ALG_LEM_TOEPLITZ}, we obtain:
\[
\norm{A^{i}}_2 \le \sum_{k=0}^{n-1} {{i} \choose{ k}}\le \paren{\frac{e i}{n-1}}^{n-1}
\]
where the second inequality is classical, see Exercise~0.0.5 in~\cite{vershynin2018high}.

Finally, we have:
\[
S_t\le t \paren{\frac{e t}{n-1}}^{n-1}=\O\paren{t^n}
\]
\end{proof}
\begin{corollary}\label{ALG_COR_System_Upper_Bound}
The norms $\norm{\O_k}_2$, $\norm{\Tm_k}_2$, $\norm{\Gamma_k}_2$ depend at most polynomially in $k$ if $\rho\paren{A}\le 1$ (they are $\O\paren{\text{poly}(k)}$). If the system is asymptotically stable they are upper bounded for all $k$ (they are $\O\paren{1}$).
\end{corollary}
\begin{proof}

	\textbf{Observability matrix:}
		Consider the sum $S_t=\sum_{i=0}^{t}\norm{A^{i}}_2$ of the previous lemma.
	We have:
	\[
	\norm{\O_k}_2\le \norm{C}_2 S_{k-1}=\O\paren{S_{k}}.
	\]
	The result follows by applying the previous lemma.	
	
	\textbf{Block Toeplitz matrix}
		Consider the sum $S_t=\sum_{i=0}^{t}\norm{A^{i}}_2$ of the previous lemma.
	By Lemma~\ref{ALG_LEM_TOEPLITZ} we have:
	\[
	\norm{\Tm_k}_2\le 1+\norm{C}_2\norm{K}_2 S_{k-2}=\O\paren{S_{k}}.
	\]
	The result follows by applying the previous lemma.
	
	\textbf{Covariance matrix}
	We have
	\[
	\norm{\Gamma_k}_2\le \norm{K\bar{R}K^*}_2\sum_{i=0}^{k-1}\norm{A^i}^2.
	\]
	The result follows by using a similar argument as in the previous lemma for $\sum_{i=0}^{k-1}\norm{A^i}^2$.
\end{proof}
\begin{lemma}[Lest non-zero singular values of $G$, $\O_p$ are increasing]\label{ALG_LEM_Singular_Value_Lower_Bound}
The $n-$th singular value of $\O_p$ is increasing with $p$:
\[
\sigma_{n}\paren{\O_{p_1}}\ge \sigma_{n}\paren{\O_{p_2}},\text{ for }p_1\ge p_2.
\]
The same is true for the $n-$th singular value of $G$:
\[
\sigma_{n}\paren{\O_{f}\K_{p_1}}\ge \sigma_{n}\paren{\O_{f}\K_{p_2}},\text{ for }p_1\ge p_2.
\]
\end{lemma}
\begin{proof}
	We only prove the first part. The other proof is similar.
	Notice that $\O_{p_1}$ can be rewritten as:
	\[
	\O_{p_1}=\matr{{c}\O_{p_2}\\M},
	\]
	for some matrix $M$.
Hence
\[
\O^*_{p_1}\O_{p_1}\succeq \O^*_{p_2}\O_{p_2}.
\]
Thus $\sigma^2_{n}\paren{\O_{p_1}}\ge \sigma^2_{n}\paren{\O_{p_2}}$.
\end{proof}


\end{document}